%% file: main.tex
\documentclass[journal]{IEEEtran}

\input{packages}

\input{notation}

\title{Temporal Logic Motion Planning with Convex Optimization via Graphs of Convex Sets}
\author{Vince Kurtz \IEEEmembership{Student Member, IEEE}
and 
Hai Lin
\IEEEmembership{Senior Member, IEEE}
\thanks{The support of the National Science Foundation (Grant No.  CNS-1830335, IIS-2007949) is gratefully acknowledged.} 
\thanks{Vince Kurtz and Hai Lin are both with the Electrical Engineering Department, University of Notre Dame, Notre Dame, IN, USA.}}

\begin{document}

\maketitle

\begin{abstract}
    Temporal logic is a concise way of specifying complex tasks. But motion planning to achieve temporal logic specifications is difficult, and existing methods struggle to scale to complex specifications and high-dimensional system dynamics. In this paper, we cast Linear Temporal Logic (LTL) motion planning as a shortest path problem in a Graph of Convex Sets (GCS) and solve it with convex optimization. This approach brings together the best of modern optimization-based temporal logic planners and older automata-theoretic methods, addressing the limitations of each: \hl{we avoid clipping and passthrough by representing paths with continuous Bezier curves}; computational complexity is polynomial (not exponential) in the number of sample points; global optimality can be certified \hl{(though it is not guaranteed}); soundness and \hl{probabilistic} completeness are guaranteed under mild assumptions; and most importantly, the method scales to complex specifications and high-dimensional systems, including a 30-DoF humanoid. Open-source code is available at \texttt{\url{https://github.com/vincekurtz/ltl_gcs}}.
\end{abstract}

\section{Introduction}

Robotic and cyber-physical systems often need to do more than avoid obstacles or regulate around a set point. To this end, temporal logics like Linear Temporal Logic (LTL), Metric Temporal Logic (MTL), and Signal Temporal Logic (STL) have grown in popularity as a compact means of expressing complex control objectives. Beyond being both expressive and concise, they are relatively easy for humans to understand, with a combination of familiar boolean operators (``and'', ``or'', ``not'') and temporal operators with intuitive names (``next'', ``always'', ``eventually''). The usefulness of temporal logic becomes most obvious in scenarios like that shown in Fig.~\ref{fig:large_door_puzzle}, where a mobile robot may not pass through a door (red) \textit{until} it has picked up a corresponding key (green) and must \textit{eventually} reach a goal (blue).

Temporal logic motion planning is a fundamentally difficult (NP-hard) problem, and is especially challenging for high degree-of-freedom (DoF) systems. The current state-of-the art is to use Mixed-Integer Convex Programming (MICP) \cite{belta2019formal}. The MICP approach is sound and complete (it will always find the optimal solution if a solution exists) but it has some significant drawbacks. In particular, standard MICP encodings introduce binary variables for each subformula at each time step. Since the worst-case complexity of MICP is exponential in the number of binary variables, MICP scales poorly with both specification complexity and the number of sample points used to represent a path \cite{belta2019formal, kurtz2021more}. 

\begin{figure}
    \begin{subfigure}{\linewidth}
        \centering
        \includegraphics[width=0.6\linewidth]{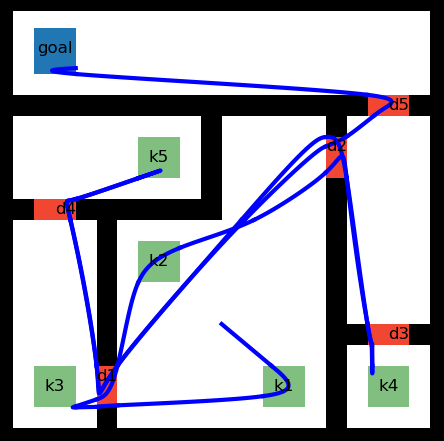}
        \caption{A benchmark from \cite{vega2018admissible}, where a mobile robot must pick up keys ($k1, k2, \dots$) before passing through doors ($d1, d2, \dots$) and reach a goal. The fastest reported solve time is \textbf{49.5 seconds} for a piecewise-linear solution \cite{sun2022multi}. Our proposed approach finds a $\mathcal{C}^2$ globally optimal solution in \textbf{5.8 seconds}.}
        \label{fig:large_door_puzzle}
    \end{subfigure}
    \vspace{0.1em}

    \begin{subfigure}{\linewidth}
        \centering
        \includegraphics[width=0.7\linewidth]{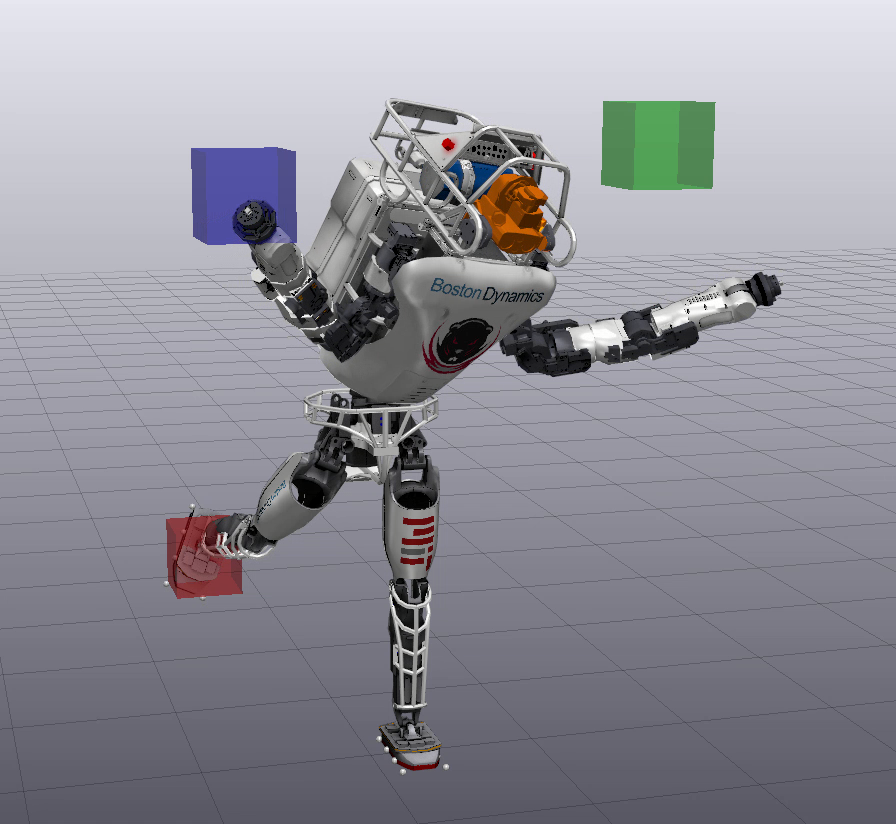}
        \caption{A 30-DoF humanoid is tasked with touching the green target with its left hand, then the red target with its right foot, then the blue target with its right hand. Convex optimization finds a smooth configuration-space solution in under 10 seconds.}
        \label{fig:atlas_cover}
    \end{subfigure}
    \caption{Our proposed approach scales to both complex task specifications (\subref{fig:large_door_puzzle}) and high-dimensional systems (\subref{fig:atlas_cover}).}
\end{figure}

The discretization of time used in standard MICP encodings compounds this problem. A coarse discretization gives rise to ``clipping'', where the path may intersect with obstacles between timesteps, violating the specification. An example of this is shown in Fig.~\ref{fig:key_door:standard}. Clipping can be mitigated with a finer discretization, but this increases computational cost exponentially \cite{belta2019formal}. 

\begin{figure}
    \begin{subfigure}{\linewidth}
        \centering
        \includegraphics[width=0.8\linewidth]{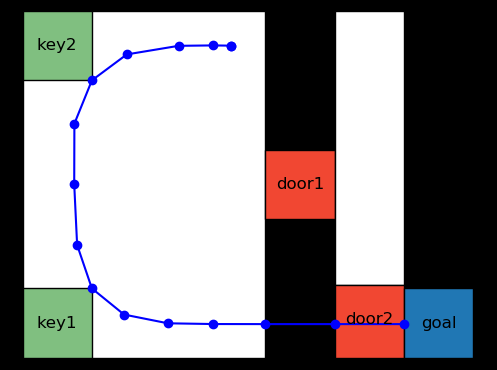}
        \caption{Standard MICP}
        \label{fig:key_door:standard}
    \end{subfigure}
    \begin{subfigure}{\linewidth}
        \centering
        \includegraphics[width=0.8\linewidth]{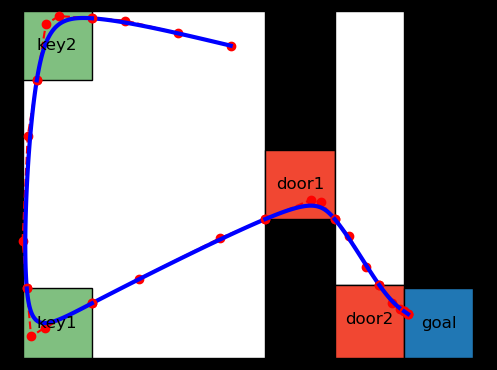}
        \caption{Our approach}
        \label{fig:key_door:ours}
    \end{subfigure}
    \caption{A mobile robot must pick up two keys before it can pass through corresponding doors to reach a goal region, as encoded in formula (\ref{eq:key_door}). A standard MICP (\ref{fig:key_door:standard}) uses a fixed discretizaiton of time, leading to a path that passes through an obstacle. Our proposed approach (\ref{fig:key_door:ours}) considers all values along a smooth Bezier spline, avoiding clipping and pass-through.}
    \label{fig:key_door}
\end{figure}

While exponential cost with the specification complexity is inevitable given the NP-hardness of the problem, it is less clear whether exponential cost with the number of sample points is also inevitable. Intuitively, planning to achieve the same specification with a finer discretization does not seem to make the problem fundamentally harder. Indeed, we show that this exponential complexity can be avoided: our proposed approach scales polynomially with the number of sample points.

Our key idea is to re-frame LTL motion planning as a shortest path problem in a Graph of Convex Sets (GCS) \cite{marcucci2021shortest}. This results in an MICP with a very tight convex relaxation---so tight, in fact, that it can often be solved to global optimality with convex optimization and rounding \cite{marcucci2022motion}. Even if the approximate solution to this MICP is sub-optimal (bounds on sub-optimality are available \cite{marcucci2022motion} and convex optimization often finds the globally optimal solution in practice), any integer-feasible solution is guaranteed to satisfy the specification.

Our proposed approach scales well to complex specifications and high-dimensional configuration spaces, outperforming the state-of-the-art on numerous benchmark problems (see Section~\ref{sec:examples}). Furthermore, we parameterize motion plans with Bezier splines, allowing us to design smooth paths without clipping or pass-through.

Our primary contributions are summarized as follows:
\begin{itemize}
    \item We show that LTL motion planning can be formulated as a shortest path problem in a GCS and solved efficiently using convex programming.
    \item \hl{By representing paths with smooth Bezier splines, we avoid the clipping and pass-through problems associated with most discrete-time temporal logic formulations.}
    \item Our proposed approach scales polynomially with the number of control points used to represent a smooth path, in contrast with the exponential complexity of standard methods. 
    \item We provide proofs of soundness and \hl{probabilistic} completeness.
    \item We demonstrate the scalability of our proposed approach to complex specifications and high-dimensional configuration spaces with several simulation examples, and provide open-source code to reproduce these results \cite{software}.
\end{itemize}

The remainder of this paper is organized as follows: Section~\ref{sec:related_work} reviews related work on temporal logic motion planning. A formal problem statement and relevant definitions are given in Section~\ref{sec:problem_formulation}. Background information on Bezier curves and graphs of convex sets is given in Section~\ref{sec:background}. Our main result---a convex programming solution to temporal logic motion planning---is described in Section~\ref{sec:ltl_as_gcs}. Section~\ref{sec:theory} provides proofs of the soundness, completeness, and computational complexity of our proposed approach. We provide examples in Section~\ref{sec:examples}, discuss limitations of our proposed method in Section~\ref{sec:limitations}, and conclude with Section~\ref{sec:conclusion}.

\section{Related Work}\label{sec:related_work}

Early work on temporal-logic-based control was dominated by automata-theoretic methods \cite{vega2018admissible,belta2017formal}. These methods generally assume that the system can be modeled as a finite-state transition system. Once a given LTL formula is transformed into an equivalent automaton, a simple graph search in the product of the transition system and the automaton reveals a satisfying path \cite[Chapter~5]{belta2017formal}. For finite-state transition systems, such automata-based methods are sound, complete, and computationally efficient. Similar techniques have achieved widespread adoption in formal verification and model checking applications \cite{baier2008principles}. 

Applying automata-based methods to motion planning is non-trivial, as we must plan in (non-finite) configuration space. The dominant idea in the literature is to use some sort of finite transition system abstraction \cite{vega2018admissible, da2019active, da2021automatic}, typically related to labeled \hl{regions in} configuration space. Once a satisfying abstract path is found, a lower-level motion planner searches for a path through the corresponding \hl{regions}. While this can be computationally efficient, as it separates logical constraints from physical dynamics, it is difficult to account for the gap between the geometry of the scenario and the abstract transition system. In practice, this means that an abstract path may be dynamically infeasible or inefficient once it is translated to a motion plan: shortest paths in the abstraction do not necessarily correspond to shortest paths in configuration space, as illustrated in Fig.~\ref{fig:shortest_path}.

\begin{figure}
    \centering
    \includegraphics[width=0.9\linewidth]{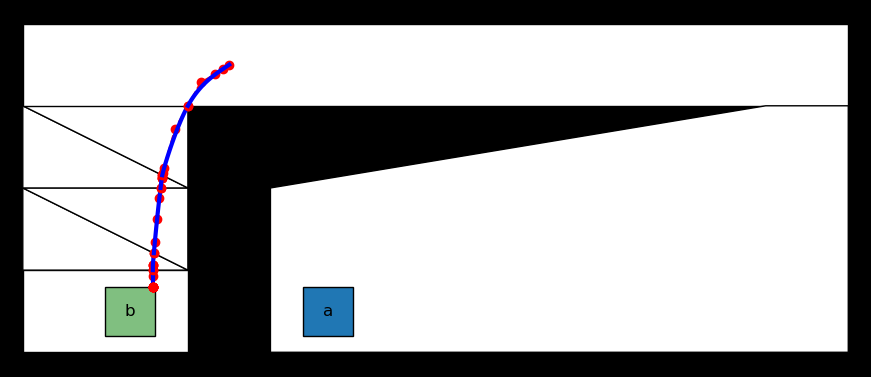}
    \caption{A mobile robot is tasked with eventually reaching region $a$ (blue) or region $b$ (green), as denoted by formula (\ref{eq:shortest_path}). Standard abstraction-based methods would likely take the longer path to $a$, since this requires passing through fewer \hl{regions}. Our proposed approach finds the shorter path to $b$ directly. }
    \label{fig:shortest_path}
\end{figure}

In prior work, we sought to address this gap with a Counter Example Guided Inductive Synthesis (CEGIS) \cite{clarke2000counterexample} inspired cycle of planning and re-planning for both a high-level automata-based planner and a low-level motion planner \cite{da2019active, da2021automatic}. While this allows for completeness guarantees, even local optimality is difficult to achieve and many re-planning cycles may be required. 

Another drawback for many automata-based methods stems from a reliance on sampling-based motion planning \cite{da2019active, vasile2013sampling,luo2021abstraction, kantaros2022perception}. While such methods can have good scalability properties, they often struggle on problems with narrow passages \cite{zhang2008efficient} and it is difficult to enforce dynamic feasibility. It is also often difficult to tell whether a specification is infeasible, as sampling-based methods do not provide certificates of infeasibility. 

\hl{
A particularly interesting set of automata-based planners do not search for specific satisfying paths, but rather generate feedback controllers that guarantee closed-loop satisfaction of the specification \cite{kloetzer2008fully,kress2009temporal}. While these methods provide an elegant way to bridge the gap between continuous dynamics and the discrete specification, their scalability is limited by the fact that they search for solutions across the whole state space. To the best of our knowledge, such techniques have not been able to scale beyond 4-dimensional state spaces. 
}

Despite these limitations, we are strongly inspired by automata-theoretic methods in this paper. Specifically, we use the GCS framework to perform graph search in the product of a transition system and an automaton. In contrast to existing abstraction-based methods, this graph search accounts for the geometry of the scenario directly, avoiding the need for separate low-level and high-level planners. Unlike sampling-based methods, our GCS method can navigate narrow passages with ease and enforce dynamical feasibility for differentially flat systems \cite{murray1995differential}. \hl{Unlike methods that search for feedback controllers, our approach scales to high-dimensional systems, including a 30-DoF humanoid.}

Optimization-based methods using mixed-integer programming were developed to address the limitations of abstraction-based methods \cite{belta2019formal}. MICP methods typically do not require any abstraction, and consider the state (or configuration) space directly. Continuous variables represent the state at discrete time steps along the path. Binary variables and additional constraints are added to enforce the specification \cite{raman2014model, belta2019formal}. The MICP approach is sound, complete, and (unlike automata-based methods) allows for guarantees of global optimality. MICP is particularly popular for STL \cite{raman2014model,kurtz2021more}, but has also been applied to LTL \cite{wolff2014optimization}, MTL \cite{kurtz2021more}, and a variety of other temporal logics. MICP is widely considered to be the state-of-the-art in temporal logic motion planning \cite{belta2019formal}.

Nonetheless, MICP methods have some significant limitations. New binary variables are introduced for each subformula (or predicate) for each time step. This means that the computational complexity is exponential not only in the size of the formula, but also in the number of time steps. Making matters worse, too few timesteps can lead to clipping and pass-through problems, where the path intersects obstacles between timesteps (Fig.~\ref{fig:key_door:standard}). 

Mitigating these limitations is an area of much active research. \cite{yang2020continuous} uses a control-barrier function between time steps to avoid clipping and pass-through. \cite{sadraddini2018formal} and \cite{kurtz2022mixed} reduce the size of the MICP, but are still left with exponential complexity in the number of time steps. \cite{sun2022multi} formulates an MICP based on piecewise linear paths, accounting for constraint satisfaction between time steps. \cite{sun2022multi} is a significant source of inspiration for us, and our use of Bezier splines to represent smooth paths is a generalization of this idea. In addition to allowing us to place constraints on the whole path, not just the sample points, Bezier splines allow us to enforce dynamic feasibility for differentially flat systems, as well as to optimize for quantities like path length, velocity, and acceleration. 

As an alternative way to make MICP more efficient, we proposed a formulation with more binary variables but a tighter convex relaxation in \cite{kurtz2021more}, resulting in better branch-and-bound solver performance in practice. The results in this paper are a substantial improvement on that basic idea. Our proposed approach results in an MICP with an even tighter convex relaxation: so tight, in fact, that it can be solved with convex optimization and rounding \cite{marcucci2022motion}. Interestingly, it is only through connections with older automata-theoretic methods that we are able to formulate this efficiently-solvable MICP. In this sense, we believe that our proposed approach brings together the best of automata-based and MICP-based methods. 

Finally, we acknowledge the recent trend of attempting to avoid the NP-hardness of temporal logic motion planning altogether by providing approximate solutions via non-convex optimization \cite{pant2017smooth,mehdipour2019arithmetic,gilpin2020smooth,kurtz2020trajectory}, learning \cite{cai2021reinforcement,LeungPavone2022}, or control barrier functions \cite{lindemann2018control, srinivasan2020control}. While such approaches can be extremely efficient and may be practical for some applications, they offer limited or no completeness guarantees and rarely scale to very complex specifications like that shown in Fig.~\ref{fig:large_door_puzzle}.

\section{Problem Formulation}\label{sec:problem_formulation}

\subsection{Linear Temporal Logic}\label{sec:problem_formulation:ltl}

In this section, we introduce the basics of LTL. Further details can be found in \cite{belta2017formal} and \cite{baier2008principles}. 

The syntax, or grammar, of LTL is defined as follows:
\begin{equation}\label{eq:ltl}
    \varphi := \true \mid a \mid \lnot \varphi \mid \varphi_1 \land \varphi_2 \mid \lnext \varphi \mid \varphi_1 \until \varphi_2,
\end{equation}
where $\true$ denotes ``true'', $a$ is an atomic proposition from the set $AP$, $\lnot$ (``not'') is the negation operator, $\land$ (``and'') is the conjunction operator, $\lnext$ is the ``next'' temporal operator, and $\until$ is the ``until'' temporal operator. These operators can be combined to form new operators like ``or'' ($\varphi_1 \lor \varphi_2 = \lnot (\lnot \varphi_1 \land \lnot \varphi_2)$), ``eventually'' ($\eventually \varphi = \true \until \varphi$), and ``always'' ($\always \varphi = \lnot \eventually \lnot \varphi$).

The semantics, or meaning, of LTL is defined over sequences of atomic propositions called \textit{words}:
\begin{definition}
    A word $\sigma = A_0, A_1, A_2, \dots$ is a sequence of atomic propositions, where $A_i \in 2^{AP}$ is a set of atomic propositions. We denote the suffix beginning at index $j$ as $\sigma[j\dots] = A_j,A_{j+1},\dots$.
\end{definition}
With this in mind, LTL semantics are defined recursively as follows, where we denote the fact that a word $\sigma$ satisfies an LTL formula $\varphi$ with $\sigma \vDash \varphi$:
\hl{
\begin{itemize}
    \item $\sigma[j\dots] \vDash \true$.
    \item $\sigma[j\dots] \vDash a$ if and only if $a \in A_j$.
    \item $\sigma[j\dots] \vDash \varphi_1 \land \varphi_2$ if and only if $\sigma[j\dots] \vDash \varphi_1$ and $\sigma[j\dots] \vDash \varphi_2$.
    \item $\sigma[j\dots] \vDash \lnot \varphi$ if and only if $\sigma[j\dots] \nvDash \varphi$.
    \item $\sigma[j\dots] \vDash \lnext \varphi$ if and only if $\sigma[j+1\dots] \vDash \varphi$.
    \item $\sigma[j\dots] \vDash \varphi_1 \until \varphi_2$ if and only if $\exists k \geq j$ such that $\sigma[k\dots] \vDash \varphi_2$ and $\sigma[i\dots] \vDash \varphi_1$ for all $j \leq i < k$.
\end{itemize}}

Full LTL, as defined above, is evaluated on infinite length words $\sigma \in (2^{AP})^\omega$. An important subset, or fragment, of LTL is the set of \textit{syntactically co-safe} formulas:
\begin{equation}\label{eq:scltl}
    \varphi := \true \mid a \mid \lnot a \mid \varphi_1 \land \varphi_2 \mid \varphi_1 \lor \varphi_2 \mid \lnext \varphi \mid \varphi_1 \until \varphi_2.
\end{equation}
Satisfaction of such a co-safe formula can be uniquely determined by a finite-length word. This fragment is of particular interest for motion planning problems, since most motion plans are of finite length. Infinite words can be relevant to motion planning, however, if plans contain loops: for example, $\varphi = \always \eventually a$ might specify that a robot must visit a recharging station labeled ``$a$'' infinitely often. In this paper, we present motion planning algorithms for both the co-safe fragment and full LTL, though we are able to provide the strongest completeness and optimality guarantees only for co-safe formulas. 

The final aspect of LTL that we will highlight here is its relationship with automata. In short, any LTL formula can be transformed into a Deterministic Buchi Automaton (DBA), and any co-safe LTL formula can be transformed into a Deterministic Finite Automaton (DFA). These are defined as follows:
\begin{definition}\label{def:automaton}
    An deterministic automaton is a tuple $\mathcal{A} = (Q, q_0, \Sigma, \delta, F)$, where
    \begin{itemize}
        \item $Q$ is a set of states,
        \item $q_0 \in Q$ is an initial state,
        \item $\Sigma$ is an input alphabet,
        \item $\delta : Q \times \Sigma \to Q$ are transition relations
        \item $F \subseteq Q$ is a set of accepting states
    \end{itemize}
    If $\mathcal{A}$ is a DFA, any accepting run must end at a state in $F$. If $\mathcal{A}$ is a DBA, any accepting run must visit $F$ infinitely often.
\end{definition}

An example of an LTL formula and a corresponding automaton is shown in Figure~\ref{fig:dfa}. Only those words that satisfy the LTL formula are accepted as inputs to the automaton. The initial state is (1), while the accepting state (2) is marked with a double circle.

\begin{figure}
    \centering
    \includegraphics[width=0.6\linewidth]{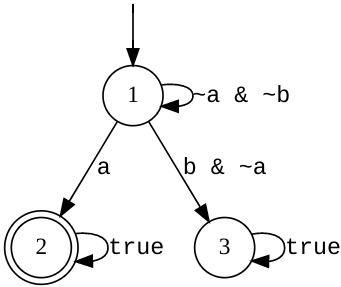}
    \caption{An automaton corresponding to the LTL formula $\varphi = \lnot b \until a$, generated by the LTLf2DFA tool \cite{ltlf2dfa}.}
    \label{fig:dfa}
\end{figure}

Details of the conversion process between LTL and DFA/DBA can be found in \cite{baier2008principles,belta2017formal}. While the worst-case complexity for this conversion is double-exponential in the size of the LTL formula, mature model-checking software is available which can perform the conversion fairly rapidly for modestly sized formulas \cite{belta2017formal,ltlf2dfa,monamanual2001}. 

\subsection{LTL Motion Planning}\label{sec:problem_formulation:ltl_motion_planning}

To apply LTL formulas to motion planning, we assume that a \hl{set of labeled convex sets in} the robot's configuration space is given. More formally, we use $\mathcal{X}$ to denote a set of convex \hl{regions} of the robot's configuration space, $\mathcal{X}_i \subseteq \mathbb{R}^n$. We also define $L : \mathcal{X} \to 2^{AP}$ as the mapping from \hl{these regions} to a set of atomic propositions. The label $L(\mathcal{X}_i)$ indicates which atomic propositions hold in \hl{region} $\mathcal{X}_i$. \hl{This label applies to every configuration in $\mathcal{X}_i$.} It it possible that no atomic propositions hold at a given \hl{region}, i.e., $L(\mathcal{X}_i) = \emptyset$, and \hl{the regions} may or may not overlap. \hl{Additionally, the set of labeled regions need not cover the entire configuration space.}

Informally, our goal is to find a path in configuration space that satisfies a given specification. More formally, we define LTL satisfaction for a path as follows:
\begin{definition}
    A path $p : \mathbb{R}^+ \to \mathbb{R}^n$ is a mapping from time $t \in \mathbb{R}^+$ to configuration\footnote{Systems where the configuration space is not $\mathbb{R}^n$ could be considered if we represent paths with generalized Bezier splines \cite{kim1995general,luo2016generalized}. It is often the case for differentially flat systems, however, that even if the configuration space is non-Euclidean, flat outputs are in $\mathbb{R}^n$ \cite{murray1995differential}.} $q \in \mathbb{R}^n$, such that $q_t = p(t)$ is the system configuration at time $t$.
\end{definition}

\begin{definition}
   The trace of path $p$ is the sequence of labels $L(\mathcal{X}_i)$ associated with the \hl{regions} $\mathcal{X}_i$ visited by $p$, i.e., $\mathrm{trace}(p) = L(\mathcal{X}_0),L(\mathcal{X}_1),\dots$.
\end{definition}
In this way, we can say that a path $p$ satisfies a specification $\varphi$ if $\mathrm{trace}(p) \vDash \varphi$.

\subsection{Problem}\label{sec:problem_formulation:problem}

With this notion of path satisfaction in mind, we present a formal problem statement as follows:

\begin{problem}\label{prob:main}
    Given an initial configuration $q_0$, \hl{convex regions} $\mathcal{X}$, labels $L$, and an LTL specification $\varphi$, find a minimum-cost path that satisfies the specification, i.e.,
    \begin{subequations}\label{eq:main_problem}
    \begin{align}
        \min_{p} ~& J(p) \\
        \mathrm{s.t.} ~& \mathrm{trace}(p) \vDash \varphi, \\
                       & p \in \mathcal{C}^d, \quad d \geq 0, \\
                       & p(0) = q_0.
    \end{align}
    \end{subequations}
\end{problem}

The cost $J$ can be any convex function of the path $p$ or its time derivatives. For example, it might be desirable to minimize path length or snap (second derivative of acceleration) \cite{mellinger2011minimum}. A variety of costs which are are convex in our parameterization of $p$ are discussed in Section~\ref{sec:ltl_as_gcs:cosafe}. The constraint $p \in \mathcal{C}^d$ ensures $d$-times continuous differentiability, allowing us to enforce dynamical feasibility for any differentially flat system \cite{murray1995differential}. 

\section{Background}\label{sec:background}

\subsection{Bezier Curves}\label{sec:background:bezier_curves}

To solve the optimization problem (\ref{eq:main_problem}) numerically, we need to somehow represent the continuous path $p$ with a finite set of points. The standard approach in the temporal logic literature is to discretize the path, considering its value only at regularly spaced sample points \cite{belta2019formal}. But, as discussed in Section~\ref{sec:related_work}, this sort of discretization leads to clipping and pass-through problems.

In this paper, we parameterize the path $p$ with a sequence of Bezier curves (i.e., a Bezier spline). A Bezier curve $b : [0,1] \to \mathbb{R}^n$ is defined by 
\begin{equation}
    b(s) = \sum_{i=0}^{k}\begin{pmatrix}k \\ i\end{pmatrix}(1-s)^{k-i}s^i \gamma_i,
\end{equation}
where $k$ is the order and $\gamma_i \in \mathbb{R}^n$ are the $k+1$ control points. Note that piecewise linear paths, such as those considered in \cite{sun2022multi}, are a special case where $k=1$.

Bezier curves have numerous appealing properties, and indeed they have been widely used in the robotic motion planning literature \cite{lau2009kinodynamic}, though they are less common in the context of temporal logic:
\begin{itemize}
    \item A Bezier curve of order $k$ is differentiable $k-1$ times.
    \item The time derivative of a Bezier curve is another Bezier curve.
    \item Bezier curves can be lined up to form a path (spline) by placing constraints on their control points, i.e., the last control point of the previous segment must be equal to the first control point of the next segment. Similar constraints can be applied to enforce a desired degree of differentiability.
    \item A Bezier curve is contained in the convex hull of its control points. 
\end{itemize}

Figures \ref{fig:key_door}, \ref{fig:shortest_path}, and \ref{fig:multitarget} show such Bezier splines. Control points are illustrated with red dots connected by dotted lines, while the path is shown in solid blue.

\subsection{Graphs of Convex Sets}\label{sec:background:gcs}

The computational workhorse behind our proposed approach is Graphs of Convex Sets (GCS), a framework first introduced in \cite{marcucci2021shortest} and applied to standard motion planning problems (reach a goal and avoid obstacles) in \cite{marcucci2022motion}. In this work, we are heavily inspired by \cite{marcucci2022motion}, and show that the promising computational attributes of GCS can be applied to motion planning problems much more complex than the classical reach-avoid problem. \hl{In particular, we show that LTL motion planning can be cast as a shortest path problem in a GCS, which can be solved efficiently using tools developed in \cite{marcucci2021shortest, marcucci2022motion}.}

In short, the GCS framework aims to find the shortest path in a graph, where each vertex of the graph is associated with a convex set, and path lengths depend on which point in the set is chosen. A simple illustration can be found in \cite[Fig.~1]{marcucci2021shortest}. More formally, a GCS is defined as follows:
\begin{definition}
    A Graph of Convex Sets $\mathcal{G} = (\mathcal{V}, \mathcal{E})$ is a directed graph where
    \begin{itemize}
        \item $\mathcal{V}$ is a set of vertices,
        \item Each vertex $v \in \mathcal{V}$ is associated with a convex set $\mathcal{X}_v$ and a point $x_v \in \mathcal{X}_v$,
        \item $\mathcal{E} \subset \mathcal{V} \times \mathcal{V}$ is a set of edges,
        \item Each edge $e = (u,v) \in \mathcal{E}$ is associated with a convex non-negative length function $l_e(x_u,x_v)$ and (optionally) a convex constraint $(x_u, x_v) \in \mathcal{X}_e$.
    \end{itemize}
\end{definition}

Given a GCS, the goal is to find a minimium-cost path from a source vertex $v_0 \in \mathcal{V}$ to a target vertex $v_T \in \mathcal{V}$. Defining a path $\xi$ as a sequence of vertices, denoting the set of all paths that connect $v_0$ and $v_T$ as $\Xi$, and denoting as $\mathcal{E}_\xi$ the set of edges traversed by path $\xi$, we can state this problem formally as follows:
\begin{subequations}\label{eq:gcs}
\begin{align}
    \min ~& \sum_{e=(u,v)\in\mathcal{E}_\xi} l_e(x_u, x_v) \\
    \mathrm{s.t.~}& \xi \in \Xi \\
                  & x_v \in \mathcal{X}_v \quad \forall v \in p \\
                  & (x_u, x_v) \in \mathcal{X}_e \quad \forall e = (u,v) \in \mathcal{E}_\xi.
\end{align}
\end{subequations}
This problem turns out to be NP-hard in general, and but an efficient MICP encoding (using perspective functions and harnessing graph-theoretic connections) was proposed in \cite{marcucci2021shortest}. This MICP has a very tight convex relaxation, meaning it is fairly close to convex optimization and tends to be efficiently solved by branch-and-bound solvers. The convex relaxation is so tight, in fact, that good solutions can often be found with a combination of convex optimization and rounding \cite{marcucci2022motion}. When convex optimization is used in this way, it is possible to bound the optimality gap without solving a full MICP \cite[Section~4.2]{marcucci2022motion}, and convex optimization is often able to find a globally optimal solution---such is the case with the example in Fig.~\ref{fig:large_door_puzzle}.

While there are many important details in the transcription of problem (\ref{eq:gcs}) into a convex program, we refer the interested reader to \cite{marcucci2021shortest, marcucci2022motion} for these details. In this paper, we use the GCS tools provided in Drake \cite{drake}. With that in mind, our primary contribution is to show that Problem~\ref{prob:main} can be rewritten as a GCS problem of the form (\ref{eq:gcs}): then existing tools can be used to solve (\ref{eq:gcs}) rapidly with convex programming. 

\section{LTL Motion Planning as a Graph of Convex Sets}\label{sec:ltl_as_gcs}

In this section, we present our main results on transforming an LTL motion planning problem into a GCS problem. Once we have a GCS problem of the form (\ref{eq:gcs}), we can solve (\ref{eq:gcs}) exactly using MICP or approximately using convex optimization and rounding. 

In Section~\ref{sec:ltl_as_gcs:cosafe} we consider co-safe LTL formulas, a fragment which includes most specifications relevant to motion planning. In Section~\ref{sec:ltl_as_gcs:full_ltl} we extend these result to full LTL formulas, considering infinite-length paths via loops. 

\subsection{Co-Safe Formulas}\label{sec:ltl_as_gcs:cosafe}

The basic idea, outlined in Algorithm~\ref{alg:cosafe}, is as follows: first we construct a finite transition-system abstraction and convert the given co-safe LTL formula into a DFA. We then construct a graph of convex sets as the product of the transition system and the DFA. The shortest path in this product GCS corresponds to a minimum-cost path that satisfies the LTL specification. 

\begin{algorithm}
    \caption{Co-Safe LTL Motion Planning}\label{alg:cosafe}
    \begin{algorithmic}
        \Require spec $\varphi$, \hl{regions} $\mathcal{X}$, labels $L$, initial state $q_0$

        \State $TS = TransitionSystem(\mathcal{X}, L)$ \Comment{Def.~\ref{def:transition_system}}
        \State $\mathcal{A} = DFA(\varphi)$
        \State $\mathcal{G} = TS \otimes \mathcal{A}$ \Comment{Def.~\ref{def:product_gcs}}
        \State $p^* = $ shortest path in $\mathcal{G}$ \Comment{MICP or convex opt.}
        
        \Return $p^*$
    \end{algorithmic}
\end{algorithm}

This procedure is very similar to, and indeed heavily inspired by, classical model checking approaches that also take a product between a DFA and a transition system, then perform graph search to certify satisfaction \cite{belta2017formal}. They key difference in this case is that the physical layout of the scenario and constraints on the continuous path are included in the product graph, which is a GCS. Existing automata-based LTL motion planning methods return a sequence of \hl{regions} from graph search, and an additional low-level planner is needed to find a continuous path consistent with this sequence of \hl{regions} \cite{vega2018admissible, da2019active, da2021automatic}. Our approach, on the other hand, combines the automata-based handling of logical constraints and optimization-based consideration of a continuous path in one step, via the GCS framework.

First, we construct a transition system from the given labeled \hl{convex regions} as follows:
\begin{definition}\label{def:transition_system}
    The transition system abstraction is a tuple $TS = (S, s_0, \to, \mathcal{L}, \mathcal{P})$, where
    \begin{itemize}
        \item $S$ is a set of states corresponding to each \hl{region},
        \item $\mathcal{P} : S \to \mathcal{X}$ is a mapping from states to convex sets in configuration space,
        \item $s_0$ is an initial state such that $q_0 \in \mathcal{P}(s_0)$,
        \item $\to$ is a transition relation, where $s \to s'$ if and only if $\mathcal{P}(s) \cap \mathcal{P}(s') \neq \emptyset$,
        \item $\mathcal{L} : S \to 2^{AP}$ is a labeling function, $\mathcal{L}(s) = L(\mathcal{P}(s))$.
    \end{itemize}
\end{definition}
Note that we define a transition whenever there is a non-empty intersection between \hl{regions}. Intersecting \hl{regions} may have significant overlap, or may simply be adjacent to one another. 

To transform the LTL formula $\varphi$ into a DFA $\mathcal{A}$, we can use any one of a variety of existing software tools. In the examples of Section~\ref{sec:examples}, we use the LTLf2DFA tool \cite{ltlf2dfa}. Further details on this conversion procedure can be found in \cite{belta2017formal,baier2008principles}. 

The final step is to formulate a GCS as the product of $TS$ and $\mathcal{A}$:
\begin{definition}\label{def:product_gcs}
    The product $\mathcal{G} = TS \otimes \mathcal{A}$ is a GCS as follows:
    \begin{itemize}
        \item $\mathcal{V} = \{S \times Q, v_T\}$ are the vertices\hl{,}
        \item $v_0 = (s_0, q_0)$\hl{,}
        \item $v_T$ is an extra vertex, where edges $(s,q) \to v_T$ exist only if $q \in F$,
        \item other edges in $\mathcal{E}$ are such that $(s,q) \to (s',q')$ exists if and only if $s \to s'$ is a transition in $TS$, and $\delta(q, \mathcal{L}(s)) = q'$ is a transition in $\mathcal{A}$,
        \item $\mathcal{X}_v = \mathcal{P}(s)^{k+1}$, where $k$ is is the desired Bezier curve degree and $\mathcal{Y}^k$ denotes the Cartesian power of convex set $\mathcal{Y}$,
        \item $x_v = [(\gamma^v_0)^T, (\gamma^v_1)^T, \dots, (\gamma^v_k)^T]^T$ are Bezier curve control points associated with vertex $v$,
        \item $\mathcal{X}_e$ are defined such that $\gamma^u_k = \gamma^v_0$, ensuring continuity of adjacent Bezier curves. Similar constraints are applied to control points for the derivatives of the Bezier curve, up to a desired degree of smoothness. 
    \end{itemize}
\end{definition}

The basic idea is to define a vertex for each unique pair of states in the DFA and the TS. Edges connect vertices only if corresponding edges exist in both the DFA and the TS: this ensures that both the physical constraints of the scenario and the logical constraints from the LTL formula are met. The target state $v_T$ is defined such that any path to $v_T$ must first pass through an accepting state ($F$) of the DFA, enforcing satisfaction of the formula. The convex set $\mathcal{X}_v$ for each vertex corresponds to a convex \hl{region}. The continuous variables $x_v \in \mathcal{X}_v$ are the control points of a Bezier curve that is constrained to lie within the corresponding \hl{region}. Edge constraints $\mathcal{X}_e$ ensure that control points of adjacent states line up, forming a continuous and smooth path.  

Defined in this way, any path from $v_0$ to $v_T$ in the GCS $\mathcal{G}$ corresponds to a smooth Bezier spline in configuration space, and is guaranteed to satisfy the given co-safe LTL specification. 

The last step is to define edge lengths $l_e(x_u, x_v)$, which corresponds to setting the path cost $J(p)$ in (\ref{eq:main_problem}). While any convex function could be used, an approximation of the Bezier curve length such as 
\begin{equation}
    l_e(x_u, x_v) = \sum_{i=0}^k\|\gamma^u_i - \gamma^u_{i+1}\|
\end{equation}
is particularly appealing. If $\|\cdot\|$ is chosen to be the L2 norm, then $J(p)$ is a strict overapproximation of the actual path length, with the approximation typically growing tighter the more control points are used. For very large-scale problems, it may be more appealing to choose $\|\cdot\|$ as the L1 norm. While this gives a looser approximation of the path length, the resulting convex programs are Linear Programs (LPs) rather than Second Order Cone Programs (SOCPs) \cite{marcucci2021shortest}, for which numerical solvers are more mature and scalability can be improved. Since the derivative of a Bezier curve is another Bezier curve, a similar procedure can be applied to add cost terms related to the time derivatives of the path.

\subsection{Full LTL}\label{sec:ltl_as_gcs:full_ltl}

To consider full LTL specifications, we need to somehow reason about infinite-length paths with a finite number of decision variables. Again inspired by the model checking literature, we choose to represent such infinite paths with loops \cite{baier2008principles}, \hl{e.g.,}
\begin{equation}
    \hl{\sigma = A_0, A_1, \dots, A_k, (B_0, B_1, \dots, B_l)^\omega.}
\end{equation}
\hl{This representation allows us to consider infinite behavior by searching for two finite paths: a prefix that is executed once and a suffix that repeats forever.}

Much of the machinery developed in Section~\ref{sec:ltl_as_gcs:cosafe} above can be applied directly to the full LTL case. We can similarly define a transition system $TS$, convert the LTL formula $\varphi$ into a DBA (rather than a DFA) $\mathcal{A}_B$, and compute a GCS as the product $\mathcal{G} = TS \otimes \mathcal{A}_B$.

The key difference is the acceptance condition for a DBA. For a DBA, it is not sufficient to simply reach an accepting state: we must both reach an accepting state and find a loop that visits the accepting state infinitely often. \hl{This loop corresponds to the second finite path, with trace $B_0, B_1, \dots, B_l$.}

We address this problem by breaking down the solution into two separate GCS solves. This procedure is outlined in Algorithm~\ref{alg:full_ltl}. The first solve is identical to the procedure for co-safe formulas outlined above: we simply find a path from the initial state to an accepting state. \hl{The trace of this path is the first finite word $A_0,A_1,\dots,A_k$.} The second GCS solve looks for a loop that starts and ends with the same Bezier curve segment as the first solve. \hl{The trace of this loop corresponds to the repeated finite word $B_0, B_1, \dots, B_l$}. Concatenating these two paths provides an infinite-length path with a loop that satisfies the specification.
\begin{algorithm}
    \caption{Full LTL Motion Planning}\label{alg:full_ltl}
    \begin{algorithmic}
        \Require spec $\varphi$, \hl{regions} $\mathcal{X}$, labels $L$, initial state $q_0$

        \State $TS = TransitionSystem(\mathcal{X}, L)$
        \State $\mathcal{A}_B = DBA(\varphi)$
        \State $\mathcal{G} = TS \otimes \mathcal{A}_B$

        \While {$\mathcal{G}$ contains accepting states}
        \State $p_1 = $ shortest path in $\mathcal{G}$ from $v_0$ to $v_T$
        \State $v_F = (s_F, q_F), q_F \in F$ is accepting state from $p_1$
        \State $p_2 = $ shortest path in $\mathcal{G}$ from $v_F$ to $v_F$
        \If {$p_2 \neq \emptyset$}
            \State \Return $p^* = p_1 \cdot p_2$
        \Else
            \State $\mathcal{G}.pop(v_F)$ \Comment{Remove $v_F$ and try again}
        \EndIf
        \EndWhile
        
        \State \Return no solution found
    \end{algorithmic}
\end{algorithm}

It may be possible that the first GCS solve returns a path from which no loops are possible, even though another initial path might admit a valid loop. In this case, we take a CEGIS-inspired approach, discarding the accepting state used in the first GCS solve and attempting to re-solve the problem. As explored in further detail in Section~\ref{sec:theory}, this procedure comes with less satisfying formal guarantees than the procedure for co-safe formulas. 

Nonetheless, this two-step procedure is able to find solutions for full LTL specifications like that shown in Fig.~\ref{fig:kl_loop}. In that example, the specification
\begin{equation}\label{eq:kl_loop}
    \varphi = \always (\eventually a \land \eventually b)  
\end{equation}
 requires a mobile robot to visit two regions infinitely often. The first segment, $p_1$, drives the robot to visit regions $a$ and $b$, reaching an accepting state of the automaton $\mathcal{A}_B$. The second segment, $p_2$, forms a loop and lines up with $p_1$ in region $b$.

\begin{figure}
    \centering
    \begin{subfigure}{0.32\linewidth}
        \centering
        \includegraphics[width=\linewidth]{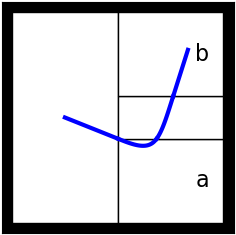}
        \caption{$p_1$}
        \label{fig:kl_loop:p1}
    \end{subfigure}
    \begin{subfigure}{0.32\linewidth}
        \centering
        \includegraphics[width=\linewidth]{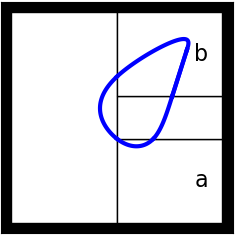}
        \caption{$p_2$}
        \label{fig:kl_loop:p2}
    \end{subfigure}
    \begin{subfigure}{0.32\linewidth}
        \centering
        \includegraphics[width=\linewidth]{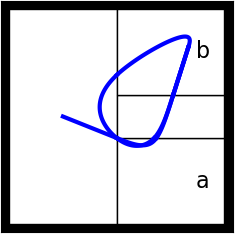}
        \caption{$p*$}
        \label{fig:kl_loop:full}
    \end{subfigure}
    \caption{A solution to the LTL specification (\ref{eq:kl_loop}), which requires the robot to visit both regions $a$ and $b$ infinitely often, and is not part of the co-safe fragment. }
    \label{fig:kl_loop}
\end{figure}

\section{Theoretical Analysis}\label{sec:theory}

In this section, we provide a formal analysis of the soundness, completeness, and computational complexity of our proposed approach. In short, our method is sound and complete for co-safe LTL formulas, sound for general LTL formulas, has double exponential complexity in the formula size, polynomial complexity in the number of control points, and polynomial complexity in the configuration space dimension. 

In all of this analysis, there are four cases to consider, since we can have co-safe or general LTL formulas, and we can solve the GCS problem with either MICP or convex optimization and rounding. 

\subsection{Soundness}

In this section we consider soundness. An LTL motion planner is sound if every motion plan returned by the planner satisfies the given specification.

Soundness is a relatively easy property to ensure, and our proposed approach is sound regardless of whether co-safe or full LTL formulas are used, and regardless of whether the GCS problem is solved with MICP or convex optimization.

\begin{theorem}[Soundness for co-safe LTL formulas]\label{theorem:cosafe_soundness}
    If Algorithm~\ref{alg:cosafe} returns a non-empty path $p$ for co-safe LTL specification $\varphi$, then $\mathrm{trace}(p) \vDash \varphi$.
\end{theorem}
\begin{proof}
    By construction of the GCS $\mathcal{G}$ (Def.~\ref{def:product_gcs}), any path in $\mathcal{G}$ that reaches the target vertex $v_T$ corresponds to a path in the DFA $\mathcal{A}$ that reaches an accepting state $q_F \in F$. By the relationship between LTL formula $\varphi$ and the DFA $\mathcal{A}$, this means that the continuous path $p$ travels through a sequence of \hl{regions} such that the corresponding sequence of labels satisfies $\varphi$. Thus the theorem holds, regardless of whether MICP or convex programming is used to find such a path. 
\end{proof}

\begin{theorem}[Soundness for full LTL formulas]\label{theorem:full_soundness}
    If Algorithm~\ref{alg:full_ltl} returns a non-empty path $p$ given LTL specification $\varphi$, then $\mathrm{trace}(p) \vDash \varphi$.
\end{theorem}
\begin{proof}
    By construction of the GCS $\mathcal{G}$, any solution to the two-step procedure of Algorithm~\ref{alg:full_ltl} corresponds to a path that reaches an accepting state $q_F \in F$ in the DBA $\mathcal{A}_B$ and loops back to $q_F$. Therefore the path visits $F$ infinitely often, satisfying the specification. Thus the theorem holds, regardless of whether MICP or convex programming is used to find such a path.
\end{proof}

It is worth noting that these soundness properties apply to the whole path $p$, not just predefined sample points. This is a notable improvement over much prior work on temporal logic motion planning in discrete time \cite{belta2019formal,belta2017formal}, where clipping between time samples might falsify the specification (e.g., lead to a collision with obstacles). Our method guarantees that each Bezier curve segment is completely contained in a single \hl{region}, eliminating this soundness gap. 

\subsection{Completeness}

In this section we consider the completeness of Algorithms~\ref{alg:cosafe} and \ref{alg:full_ltl}. An LTL motion planner is complete if it always finds a solution when one is available. \hl{The completeness guarantees in this section are posed over a given set of convex regions: optimal decomposition of configuration space into labeled convex regions remains an important area for future research.}

We first consider the case of a co-safe LTL specification where MICP is used to solve the GCS subproblem. This case offers the strongest completeness guarantees under relatively modest assumptions:

\begin{theorem}[Completeness for co-safe LTL with MICP]\label{theorem:cosafe_micp_completeness}
    Assume that there exists a satisfying Bezier spline solution of the desired degree and smoothness, and that the control points of each segment are each contained within a single \hl{region}. If MICP is used to solve the GCS subproblem, then Algorithm~\ref{alg:cosafe} will return a path that satisfies the given co-safe LTL specification.
\end{theorem}

\begin{proof}
    By construction of the GCS $\mathcal{G}$ (Def.~\ref{def:product_gcs}), if a satisfying path exists then there exists a path from $v_0$ to $v_T$ in $\mathcal{G}$. This follows from the relationship between the specification $\varphi$ and the automaton $\mathcal{A}$. By \cite[Theorem~5.6]{marcucci2021shortest}, the MICP used to solve (\ref{eq:gcs}) is guaranteed to find the optimal path from $v_0$ to $v_T$. 
\end{proof}

The primary assumption required for this theorem merits some discussion. It may be the case that a satisfying path exists, but it cannot be represented by a Bezier spline of the desired degree and smoothness. It could also be the case that a Bezier spline solution exists, but requires control points that lie outside a \hl{region}. In either of these cases, Algorithm~\ref{alg:cosafe} may fail to find a solution even though one exists. 

Fortunately, there is a simple practical fix to this problem: increase the number of control points, and/or reduce the desired smoothness. Increasing the number of control points adds additional flexibility, and typically brings control points closer to the path. Reducing the smoothness (e.g. $\mathcal{C}^1$ rather than $\mathcal{C}^2$) eliminates some of the constraints on the control points, and may similarly allow for solutions that were previously unavailable. 

In the case of co-safe LTL specifications where convex optimization and rounding are used to solve the GCS subproblem, we can certify probabilistic completeness:

\begin{theorem}[Completeness for co-safe LTL with Convex Optimization]\label{theorem:cosafe_convex_completeness}
    Assume that there exists a satisfying path made up of Bezier curve segments of desired degree and continuity, and that the control points of each segment are contained within a single \hl{region}. \hl{Furthermore, assume that the convex relaxation of the GCS subproblem assigns nonzero values to binary variables on the optimal path.} If the convex optimization and rounding scheme described in \cite{marcucci2022motion} is used to solve the GCS subproblem, then Algorithm~\ref{alg:cosafe} will return a path $p$ that satisfies the given co-safe LTL specification with probability 1 as the number of rounding iterations approaches infinity.
\end{theorem}

\begin{proof}
    By construction of the GCS $\mathcal{G}$ (Def.~\ref{def:product_gcs}), if a satisfying path exists then there exists a path from $v_0$ to $v_T$ in $\mathcal{G}$. This follows from the relationship between the specification $\varphi$ and the automaton $\mathcal{A}$. The rounding scheme described in \cite[Section~4.2]{marcucci2022motion} treats the values of binary variables from the convex relaxation as probabilities, and performs rounding as a randomized depth-first search with backtracking. As the number of rounding iterations approaches infinity, this procedure will eventually explore all possible paths in $\mathcal{G}$ with probability 1. If the convex relaxation of (\ref{eq:gcs}) is not feasible, then the MICP is not feasible and no solution exists (by Theorem~\ref{theorem:cosafe_micp_completeness}).
\end{proof}

\begin{remark}
    \hl{The assumption of nonzero values for binary variables on the optimal path is minimally restrictive in practice. In fact, branch-and-bound MICP solvers often prune branches based on integer-valued binary variables in the convex relaxation \cite{conforti2014integer}.}
\end{remark}

In practice (see the examples in Section~\ref{sec:examples}), the gap between convex optimization and MICP is quite small. For all of the examples considered in this paper, convex optimization was able to find a satisfying solution---if not a certified globally optimal solution---with only a few ($<10$) rounding iterations.

For full-LTL specifications, the necessity of finding a loop makes any completeness guarantees significantly more limited. Nonetheless, we can guarantee completeness under a set of more restrictive assumptions:

\begin{theorem}\label{theorem:full_micp_completeness}
    Assume that the conditions for Theorem~\ref{theorem:cosafe_micp_completeness} hold. Furthermore, assume that for at least one accepting state $v = (s, q), q \in F$ in the GCS $\mathcal{G}$, every Bezier curve $b$ contained in $X_v$ admits a loop through $\mathcal{G}$ that starts and ends with $b$. Then Algorithm~\ref{alg:full_ltl} will return a satisfying path. 
\end{theorem}

\begin{proof}
    By Theorem~\ref{theorem:cosafe_micp_completeness}, Algorithm~\ref{alg:full_ltl} will eventually find a path $p_1$ that reaches the special accepting state $v$. If Algorithm~\ref{alg:full_ltl} finds a path $p_1$ to an accepting state $v'$ that does not admit a loop for every Bezier curve in $X_{v'}$, then $v'$ will be removed from the graph and a new $p_1$ computed. Once a suitable $p_1$ is found, a loop $p_2$ is similarly guaranteed to be found by Theorem~\ref{theorem:cosafe_micp_completeness} and the assumption that a loop exists. The complete path $p* = p_1 \cdot p_2$ corresponds to visiting the accepting set $F$ infinitely often, thus the Theorem holds. 
\end{proof}

Probabilistic completeness guarantees can be obtained for the full LTL case when convex optimization is used instead of MICP, following similar reasoning to Theorem~\ref{theorem:cosafe_convex_completeness}.

The extra assumption that every Bezier curve in at least one accepting state admit a loop is needed to ensure that the CEGIS-inspired procedure in Algorithm~\ref{alg:full_ltl} eventually returns a solution. The assumption that \textit{every} Bezier curve admit a loop is particularly restrictive---it may be the case that a loop ($p_2$) is possible, but not from the starting Bezier curve segment that ends $p_1$. This restriction is necessary because the GCS framework does not allow for constraints between continuous variables $x_v$ that are not connected by an edge. This means that to find a loop, we must constrain the initial and final continuous variables $x_0$ and $x_T$ to some a-priori fixed values. 

\subsection{Complexity}\label{sec:theory:complexity}

In this section, we formally characterize the computational complexity of our proposed approach, focusing on the case where the GCS subproblem is solved with convex optimization and rounding, as described in \cite{marcucci2022motion}. 

At first glance, it may appear that since LTL motion planning is NP-hard, convex optimization is solvable in polynomial time, and we solve LTL motion planning with convex optimization, we have shown that $P=NP$. We emphasize that this is not the case for two reasons: first, convex optimization provides only an approximate solution to the GCS subproblem, which is itself NP-hard \cite{marcucci2021shortest}. Second, and more importantly, the size of the automaton $\mathcal{A}$ is double exponential in the size of the LTL formula $\varphi$, resulting in a GCS $\mathcal{G}$ which is also double exponential in the size of the LTL formula. 

\begin{theorem}\label{theorem:formula_size_complexity}
    \hl{Given the GCS $\mathcal{G}$, the time complexity of finding an approximate solution to the shortest path problem with convex optimization} is $O(2^{2^{|\varphi|}})$, where $|\varphi|$ is the size of $\varphi$.
\end{theorem}
\begin{proof}
    The conversion of $\varphi$ to automaton $\mathcal{A}$ results in an automaton with $O(2^{2^{|\varphi|}})$ states \cite{belta2017formal}, so the GCS $\mathcal{G}$ has $O(2^{2^{|\varphi|}})$ vertices and edges. The convex optimization encoding of \cite{marcucci2022motion} results in an LP or SOCP with variables for each vertex and edge in $\mathcal{G}$. Since LP and SOCP have polynomial time complexity in the number of decision variables \cite{nesterov1994interior}, the overall time complexity is $O(2^{2^{|\varphi|}})$.
\end{proof}

While exponential complexity in the formula size is inevitable due to the NP-hardness of the problem, our proposed approach scales polynomially in a number of variables of practical interest:

\begin{theorem}\label{theorem:control_point_complexity}
    \hl{Given the GCS $\mathcal{G}$, the time complexity of finding an approximate solution to the shortest path problem with convex optimization} is polynomial in $k$, the degree of each Bezier curve segment.
\end{theorem}
\begin{proof}
    This follows from the fact that the number of decision variables increases linearly with $k$, and the convex optimization problem is an LP or SOCP which can be solved in polynomial time with interior point methods \cite{nesterov1994interior}.  
\end{proof}

This means that the complexity is polynomial in the total number of control points $\gamma_i$ used to represent a curve. This is in contrast to standard MICP-based methods, which scale exponentially in the number of sample points used to represent a path. In practice, we find that increasing the number of control points has little practical impact on solve times, see Fig.~\ref{fig:timestep_scalability}.

\begin{theorem}\label{theorem:configuration_space_complexity}
    \hl{Given the GCS $\mathcal{G}$, the time complexity of finding an approximate solution to the shortest path problem with convex optimization} is polynomial in $n$, the dimension of the configuration space.
\end{theorem}
\begin{proof}
    This follows from the fact that the number of decision variables increases linearly with $n$, and the convex optimization problem is an LP or SOCP which can be solved in polynomial time with interior point methods \cite{nesterov1994interior}.  
\end{proof}

This is a key factor in enabling our approach to scale to high-dimensional configurations spaces, as explored in Section~\ref{sec:examples:high_dof}.

\section{Examples}\label{sec:examples}

In this section, we demonstrate the scalability of our proposed approach with several simulation examples. Code for reproducing these examples is available at \cite{software}. All experiments were performed on a laptop with an Intel i7 CPU and 32GB RAM. The underlying convex optimization solver was MOSEK \cite{mosek}, accessed via Drake \cite{drake} Python bindings.

Solve times for all of the examples are shown in Table~\ref{tab:solve_times}. In addition to the final solve time, we report the time converting the LTL specification $\varphi$ to a DFA $\mathcal{A}$ and time to form the GCS $\mathcal{G} = \mathcal{A} \otimes TS$. Converting LTL to DFA can be particularly expensive for complex specifications due to the double exponential complexity of this operation \cite{belta2017formal}. 

Of these steps, only the final solve is an online operation: all other steps can be performed offline and their results reused for different initial conditions. In all of the examples, we use convex optimization rather than MICP to solve the GCS problem.

\begin{table}[h]
    \centering
    \begin{tabular}{c|c||c|c|c}
         Name & Fig. & LTL to DFA & Form GCS & \textbf{Solve}   \\
         \hline
         shortest path & \ref{fig:shortest_path} & 0.13 & 0.05 & \textbf{0.34} \\
         simple key-door & \ref{fig:key_door} & 0.54 & 0.34 & \textbf{0.29} \\
         complex key-door & \ref{fig:large_door_puzzle} & 32 min. & 547 & \textbf{5.82} \\
         multi-target & \ref{fig:multitarget} & 0.64 & 3.72 & \textbf{7.28} \\
         manipulator & \ref{fig:robot_arm} & 0.10 & 0.08 & \textbf{1.09} \\
         humanoid & \ref{fig:atlas_cover} & 0.11 & 0.18 & \textbf{7.72} \\
    \end{tabular}
    \caption{Timing breakdown for each of the examples in this section. Times are in seconds unless otherwise noted. All operations can be performed offline except ``Solve''. }
    \label{tab:solve_times}
\end{table}

\subsection{Planar Motion Planning}\label{sec:examples:planar}

In this section, we consider motion planning for a simple planar robot. As we consider path planning rather than trajectory planning, the particular dynamics of the robot are not used. We can, however, design paths with a desired degree of smoothness (see Section~\ref{sec:background:bezier_curves}), ensuring feasibility for any differentially flat system \cite{murray1995differential}. 

The first example, shown in Fig.~\ref{fig:shortest_path}, illustrates the advantage of considering both logical constraints and scenario geometry at the same time (something that most automata-based methods struggle with). A mobile robot is tasked with eventually reaching one of two target regions, labeled $a$ and $b$ respectively:
\begin{equation}\label{eq:shortest_path}
    \varphi = \eventually (a \lor b).
\end{equation}
An abstraction-based motion planner such as \cite{da2021automatic} would likely choose the longer path to region $a$, since this requires fewer transitions in $\mathcal{G}$. Our planner, on the other hand, finds the shorter path to target $b$ directly.

The cost in this example is a simple penalty on path length (L2 norm of distance between control points). A Bezier curve of order $k=4$ is used, and $\mathcal{C}^2$ continuity is enforced. Bezier control points are shown as red dots in Fig.~\ref{fig:shortest_path}, while the final path itself is in blue. 

In the second planar motion example, illustrated in Fig.~\ref{fig:key_door}, the robot must pick up two keys (i.e., visit green regions labeled $key1$ and $key2$) in order to pass through corresponding doors (red, labeled $door1$ and $door2$) before eventually reaching a goal (blue, labeled $goal$). More formally, this specification can be written as an LTL formula
\begin{equation}\label{eq:key_door}
    \varphi = (\lnot door1 \until key1) \land (\lnot door2 \until key2) \land \eventually goal.
\end{equation}

A more complicated version of this specification with five keys and five doors is shown in Fig.~\ref{fig:large_door_puzzle}. This benchmark example was first proposed in \cite{vega2018admissible}, where it was solved with an abstraction-based approach, though no solve times were reported. To the best of our knowledge, the fastest reported solve time for this benchmark is 49.5 seconds from \cite{sun2022multi}. \cite{sun2022multi} finds a piecewise linear path (avoiding clipping and pass-through) with MICP. Our proposed approach is roughly an order of magnitude faster (Table~\ref{tab:solve_times}), provided conversion from LTL to DFA and construction of the GCS is performed offline. Furthermore, the path generated with our approach is twice continuously differentiable and certified globally optimal. The extremely long time to convert from LTL to DFA is indicative of the double-exponential complexity of this operation \cite{belta2017formal}. 

We use the simpler key-door example (with two doors) to compare with standard MICP-based temporal logic motion planners, which consider a discretization of time rather than continuous-valued curves for motion planning. In particular, we compare with the MICP approach of \cite{raman2014model} as implemented in the \textit{stlpy} \cite{kurtz2022mixed} python library\footnote{While this library is for STL rather than LTL, the specification in question does not have active timing constraints, rendering the STL and LTL specifications equivalent.}. 

The standard MICP represents continuous paths with a pre-specified number of time steps. If too few time steps are used, clipping and pass-through problems can cause the path to intersect with obstacles between time steps, as illustrated in Fig.~\ref{fig:key_door:standard}. Our proposed approach, on the other hand, considers the whole curve, and therefore does not suffer from clipping or pass-through even if relatively few control points are used, as shown in Fig.~\ref{fig:key_door:ours}.

Typically, clipping and pass-through are addressed by sampling the path more finely, using more time steps. Unfortunately, this leads to an exponential blowup in time-complexity for the standard approach. Our approach, on the other hand, scales only polynomially with the number of control points (Theorem~\ref{theorem:control_point_complexity}). The practical impact of this fact is illustrated in Fig.~\ref{fig:timestep_scalability}. 

In Fig.~\ref{fig:timestep_scalability}, we compare solve times for the standard MICP and our proposed approach, varying the number of time steps for the standard approach and the total number of control points used (the rough equivalent of time steps) for our method. Solve times for the standard MICP increase exponentially, as expected: using more than about 30 timesteps is impractical for this specification. Solve times for our proposed approach, on the other hand, stay roughly constant as the total number of control points increases. We can use as many as 90 control points---10 per Bezier curve segment in this example---with minimal impact on solve time. 

\begin{figure}
    \centering
    \includegraphics[width=\linewidth]{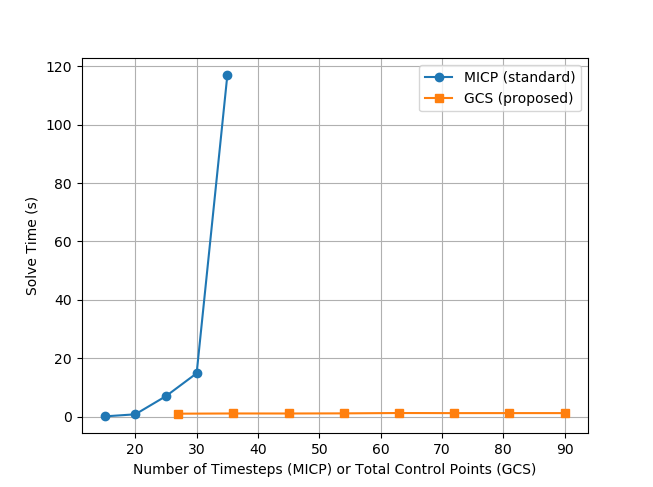}
    \caption{Plot of solve times as the number of sample points used to represent a path (number of timesteps for a standard MICP, number of control points for our approach) increases. The standard approach scales exponentially in the number of time steps, while solve times for our proposed approach are nearly constant as the number of control points grows. }
\label{fig:timestep_scalability}
\end{figure}

\begin{figure}
    \centering
    \includegraphics[width=0.8\linewidth]{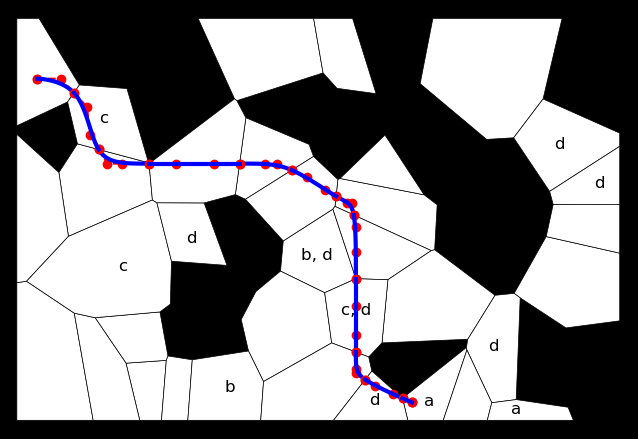}
    \caption{A mobile robot must navigate an environment with multiple labeled targets, eventually reaching $a$, $c$, and $d$, but always avoiding $b$. This is formalized by specification (\ref{eq:multitarget}).}
    \label{fig:multitarget}
\end{figure}

Our final planar motion planning example, illustrated in Fig.~\ref{fig:multitarget}, showcases a more complex scenario as well as a specification that is not part of the co-safe fragment. The robot is tasked with eventually visiting \hl{regions} labeled $a$, $c$, and $d$, while always avoiding $b$:
\begin{equation}\label{eq:multitarget}
    \varphi = \eventually a \land \always \lnot b \land \eventually c \land \eventually d.
\end{equation}

The always operator ($\always$) means that this specification is not syntactically co-safe \cite{belta2017formal}. The final loop which renders infinite-length behaviors is trivial, however: the robot simply remains at the final configuration. 

The GCS for this multi-target example is relatively large, and the shortest-path subproblem is a large and possibly poorly conditioned SOCP if an L2 norm approximation of the path length is used. We found that MOSEK \cite{mosek} was unable to converge to a solution to the SOCP in this case, returning an \texttt{Solver Status: UNKNOWN} error code. To avoid this, we used an L1 norm approximation of the path length (see Section~\ref{sec:ltl_as_gcs:cosafe}) instead. This leads to a LP rather than a SOCP, for which solvers are more mature and reliable. With this L1/LP formulation, MOSEK finds the solution shown in Fig.~\ref{fig:multitarget}.

\subsection{Comparison with Sampling-Based Motion Planning}\label{sec:examples:sampling}

For simple reach-avoid problems (a special case of LTL motion planning), GCS-based motion planning was shown to outperform sampling-based methods in \cite{marcucci2022motion}. In this section, we explore whether similar advantages hold for motion planning under LTL specifications. 

We compare our proposed approach with the state-of-the-art sampling-based planner described in \cite{luo2021abstraction}. This method is especially appealing because it does not require computing the product of a transition system and the specification automaton, but rather incrementally builds trees to explore the state-space of the product without constructing it explicitly. 

In particular, we consider the example in \cite[Section~VII.A.]{luo2021abstraction}. This example (\cite[Fig. 6]{luo2021abstraction}) has with six regions of interest ($r_i$) and two obstacles ($obs$). The task is specified by the LTL formula
\begin{multline}\label{eq:sampling_comparison}
    \varphi = \eventually (r_1 \land \eventually r_3) \land (\lnot r_1 \until r_2) \land \\ \eventually(r_5 \land \eventually(r_6 \land \eventually r_4)) \land (\lnot r_4 \until r_5) \land \always \lnot obs.
\end{multline}
This specification requires a mobile robot to do the following: eventually visit region $r_1$, then $r_3$; not visit $r_1$ until visiting $r_2$; eventually visit $r_5$, $r_6$, and $r_4$ in that order; not visit $r_4$ before $r_5$; and avoid obstacles. 

To compare our approach and \cite{luo2021abstraction}, we randomly generated 100 initial positions and solved the motion planning problem using our approach and sampling-based planning. For the sampling-based planner we used a maximum of $n_{max} = 1000$ iterations and default parameters from the implementation provided with \cite{luo2021abstraction}. For our approach, we solved for a first order Bezier spline with an L2 norm penalty on path length and performed a single rounding iteration. A histogram of solve times is shown in Fig.~\ref{fig:sampling_comparison}. For both approaches, we exclude the time needed to construct an automaton, as this can be performed offline. 

\begin{figure}
    \centering
    \includegraphics[width=\linewidth]{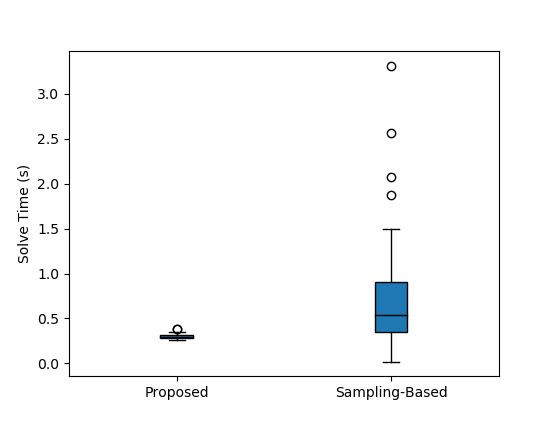}
    \caption{Solve times for specification (\ref{eq:sampling_comparison}) using our proposed approach and a sampling-based approach \cite{luo2021abstraction} over 100 trials with different initial conditions.}
    \label{fig:sampling_comparison}
\end{figure}

Solve times for our proposed approach were both faster on average (mean of 0.30 vs 0.69 seconds) and more consistent (standard deviation of 0.024 vs 0.52 seconds) than the sampling based approach, though there were some initial conditions for which the sampling-based method found a solution faster. 

\subsection{High-DoF Systems}\label{sec:examples:high_dof}

Our proposed approach scales well to high-dimensional configuration spaces. To demonstrate this, we consider the 7-DoF model of a Franka Panda arm shown in Fig.~\ref{fig:robot_arm} and the 30-DoF model of an Atlas humanoid shown in Fig.~\ref{fig:atlas_iris_seeds}. Both models are from Drake \cite{drake}.

\begin{figure*}
    \centering
    \includegraphics[width=0.16\linewidth]{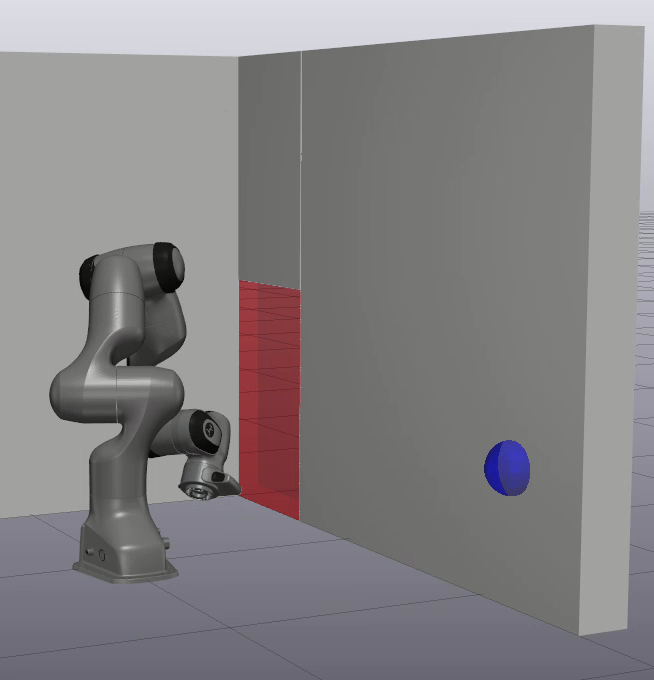}
    \includegraphics[width=0.16\linewidth]{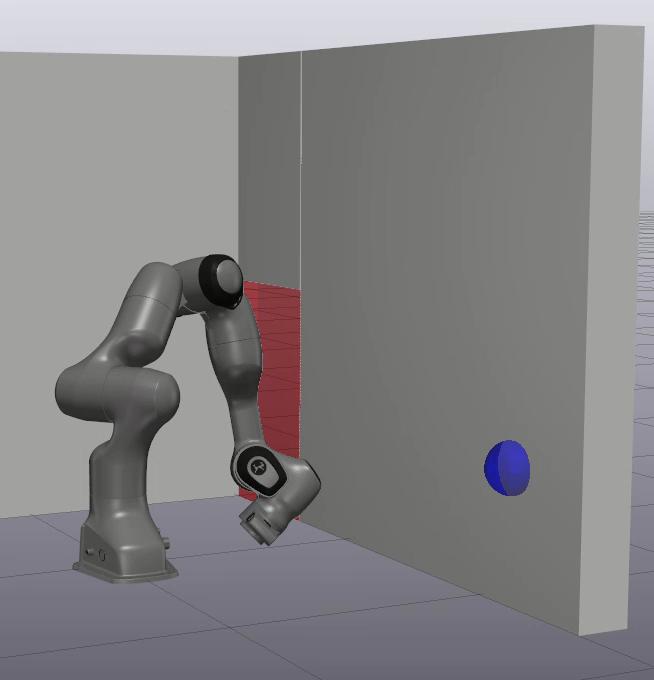}
    \includegraphics[width=0.16\linewidth]{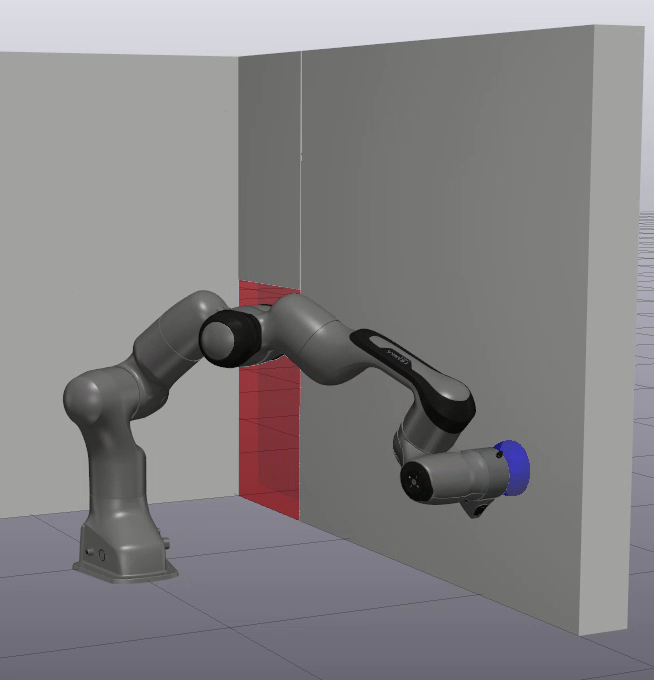}
    \includegraphics[width=0.16\linewidth]{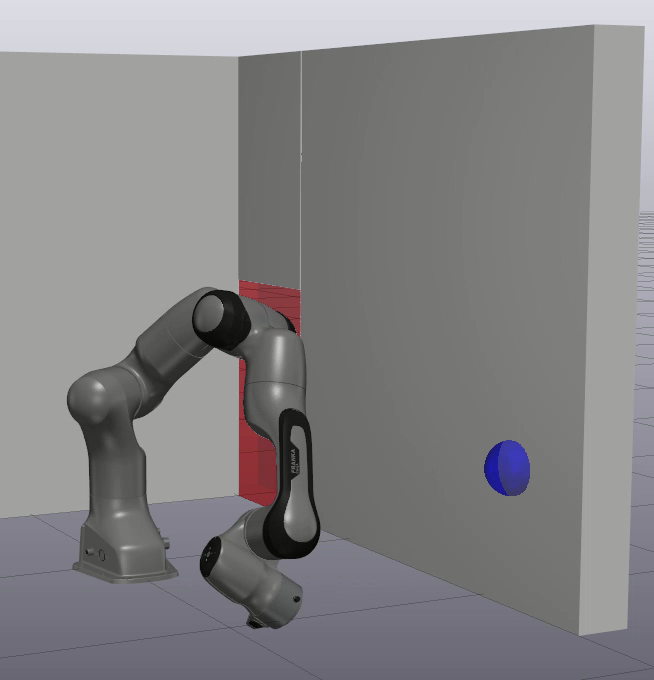}
    \includegraphics[width=0.16\linewidth]{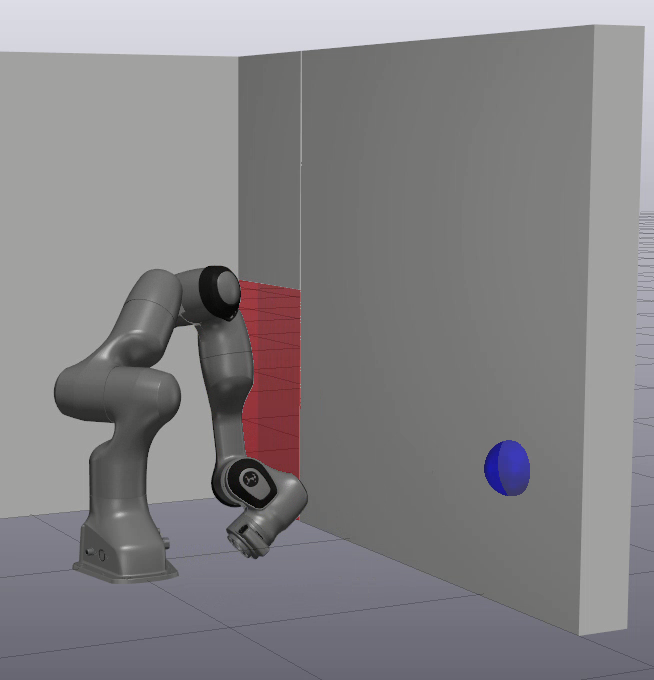}
    \includegraphics[width=0.16\linewidth]{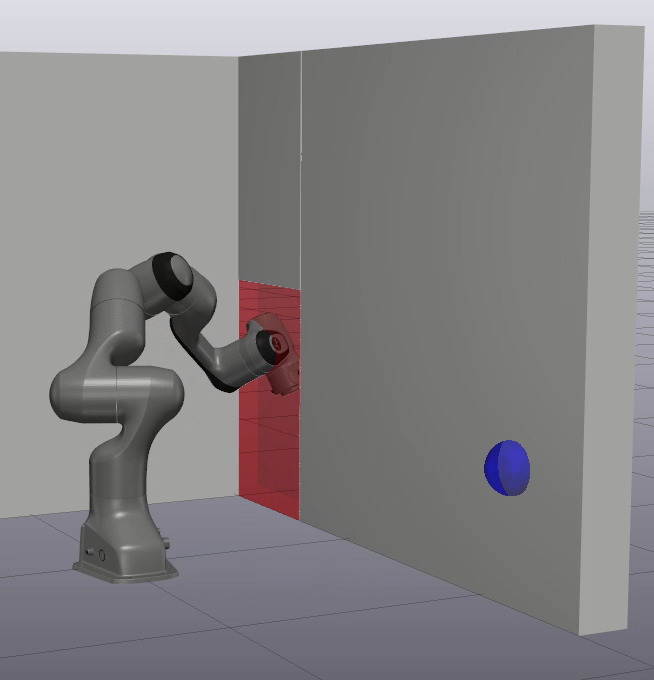}
    \caption{A 7-DoF manipulator arm is tasked with pushing a button (blue) before reaching through a doorway (red), according to formula (\ref{eq:robot_arm}). A labeled convex decomposition of the configuration space was generated offline using C-IRIS \cite{amice2023finding}.}
    \label{fig:robot_arm}
\end{figure*}

\subsubsection{7-DoF Manipulator}

The arm is tasked reaching a target configuration that places its end-effector through a doorway (red box), but it must press a button (blue half circle) before passing through the doorway. This specification can be written as 
\begin{equation}\label{eq:robot_arm}
    \varphi = \eventually target \land \lnot doorway \until button.
\end{equation}
Labeled collision-free convex polytopes in configuration space were generated using C-IRIS \cite{amice2023finding}, as implemented in Drake \cite{drake}. 

Our proposed approach offers considerable performance improvements over state-of-the-art methods for temporal logic motion planning in large configuration spaces. In prior work \cite{kurtz2020trajectory}, we showed how non-convex trajectory optimization could be used to solve temporal logic motion planning for a similar 7-DoF robot arm, and solved a relatively simple specification ($\eventually (a \lor b)$) in 160 seconds. In contrast, the considerably more complex specification (\ref{eq:robot_arm}) is solved by our proposed approach in about a second.

\subsubsection{30-DoF Humanoid}

The Atlas humanoid is tasked with reaching three target boxes, shown in green, blue, and red in Fig.~\ref{fig:atlas_cover}. We use inverse kinematics and C-IRIS \cite{amice2023finding} to define labeled convex \hl{regions} of configuration space corresponding to (a) the left hand touching the green target, (b) the right foot touching the red target, and (c) the right hand touching the blue target. We also define an unlabeled \hl{convex region} without kinematic constraints (d). Seed configurations used to generate these convex sets are shown in Fig.~\ref{fig:atlas_iris_seeds}.

\begin{figure}
    \centering
    \begin{subfigure}{0.48\linewidth}
        \centering
        \includegraphics[width=\linewidth]{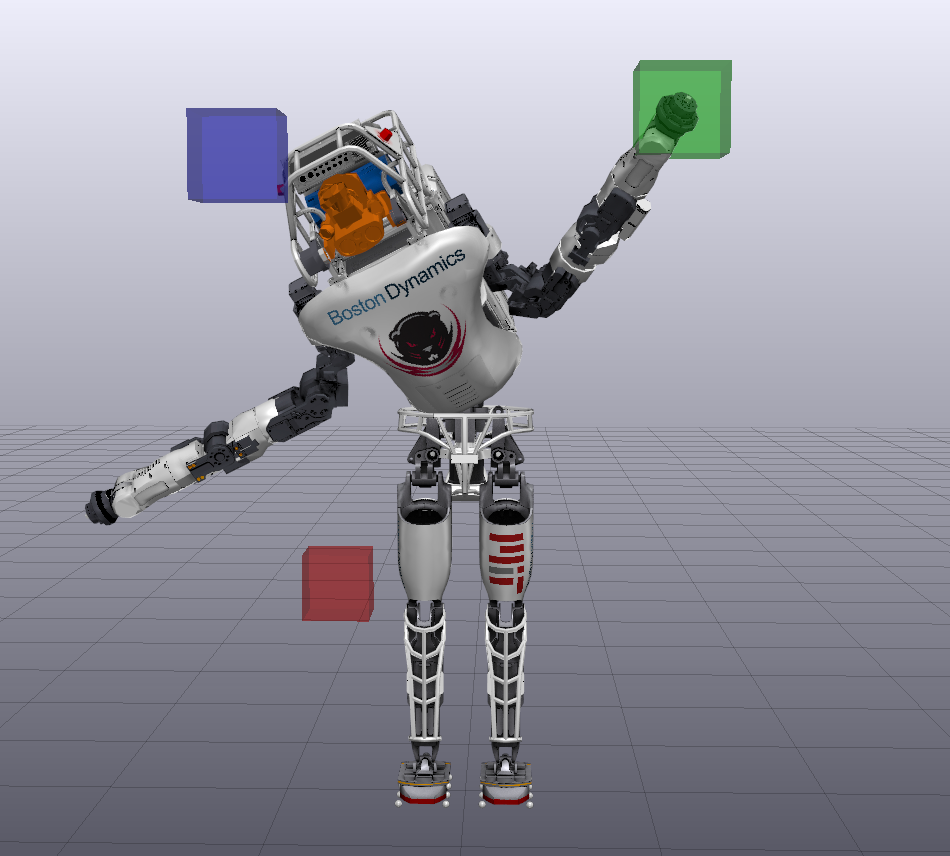}
        \caption{Label: $a$}
    \end{subfigure}
    \begin{subfigure}{0.48\linewidth}
        \centering
        \includegraphics[width=\linewidth]{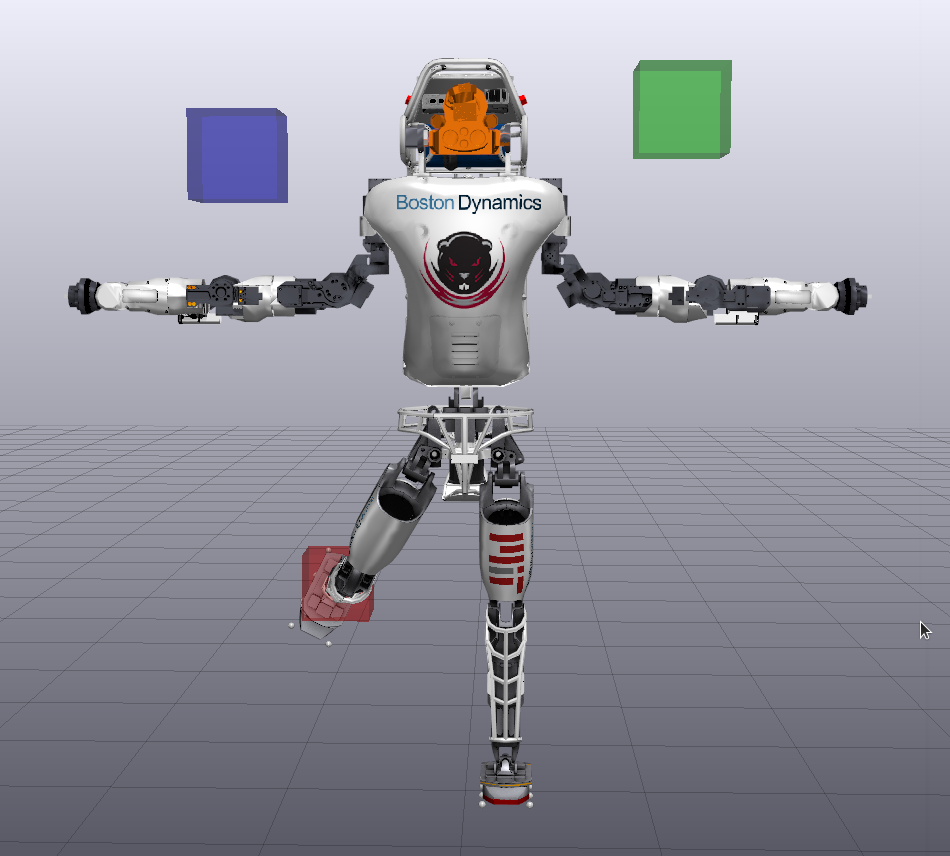}
        \caption{Label: $b$}
    \end{subfigure}
    \begin{subfigure}{0.48\linewidth}
        \centering
        \includegraphics[width=\linewidth]{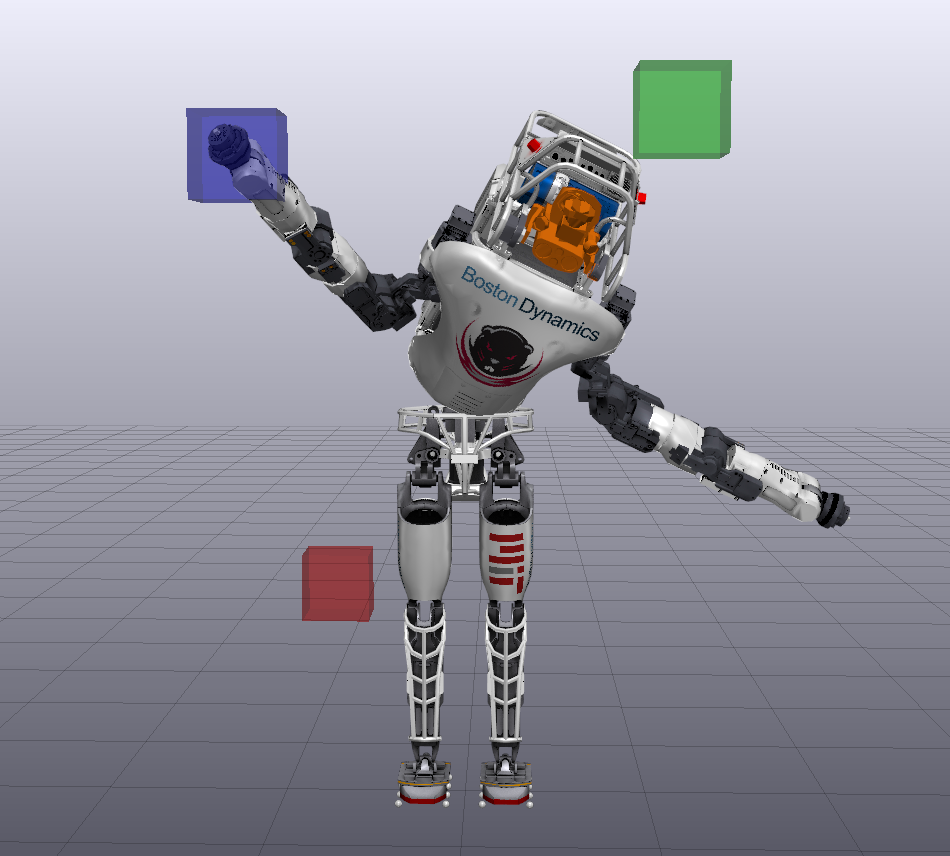}
        \caption{Label: $c$}
    \end{subfigure}
    \begin{subfigure}{0.48\linewidth}
        \centering
        \includegraphics[width=\linewidth]{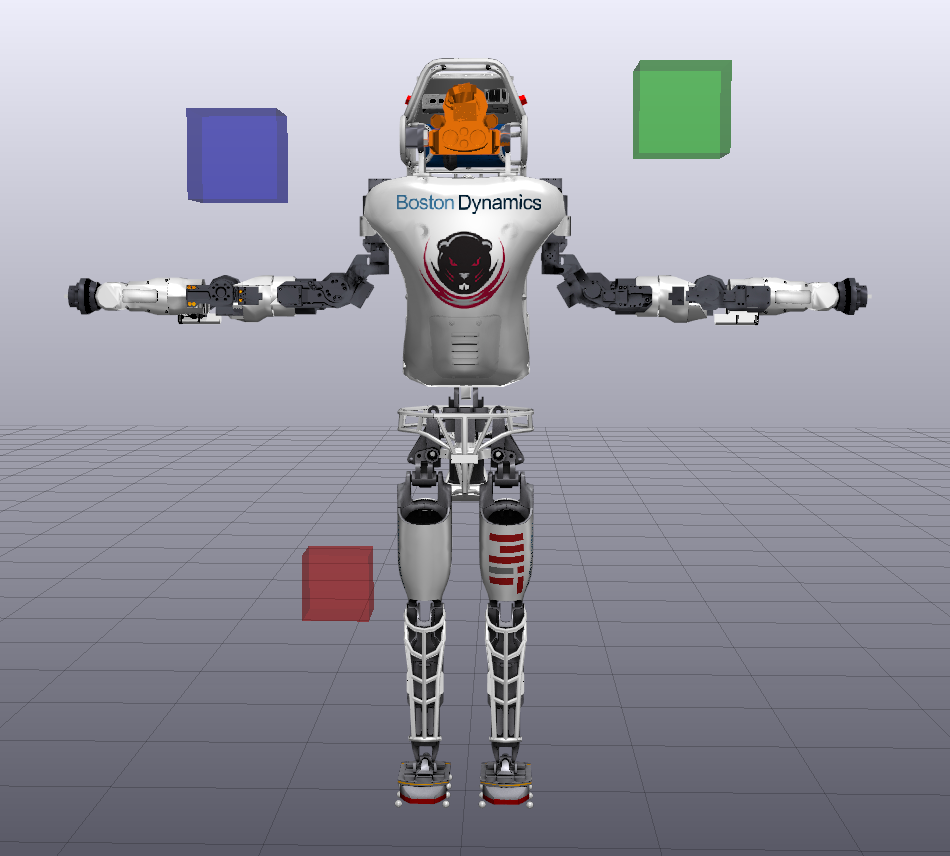}
        \caption{Label: $\emptyset$}
    \end{subfigure}
    \caption{Seed configurations used to generate labeled convex \hl{regions} for the Atlas humanoid. For \hl{regions} labeled $a$, $b$, and $c$, inverse kinematics constraints ensure that the hand or foot stays within the target box.}
    \label{fig:atlas_iris_seeds}
\end{figure}

The specification is then given by
\begin{equation}\label{eq:atlas}
    \varphi = \eventually (a \land \eventually (b \land \eventually c)),
\end{equation}
which states that the robot should eventually reach targets $a$, $b$, and $c$, in that order. 

We fix the pose of the pelvis in the world frame, and do not consider contact interactions with the ground. Balance, contact wrench cone, or other dynamic feasibility constraints are not considered: our goal is merely to find a configuration space path that satisfies the specification. Without a floating base, this is a 30-DoF system. 

We search for a satisfying path using Bezier splines of degree $k=3$ and enforce $\mathcal{C}^2$ smoothness. In addition to an L2 norm penalty on path length, we add a small penalty on configuration space accelerations. As shown in Table~\ref{tab:solve_times}, a locally optimal path is found in 7.72 seconds. 

\section{Limitations}\label{sec:limitations}

In this section, we provide a brief overview of the limitations of our proposed approach along with concrete suggestions for mitigating these limitations in practice.

A major drawback is the double-exponential complexity of converting LTL formulas to automata \cite{belta2017formal}. This limitation is illustrated most clearly in Table~\ref{tab:solve_times}, where conversion to a DFA takes over half an hour for the complex key-door scenario. We performed this conversion using the LTLf2DFA python library \cite{ltlf2dfa}, which uses MONA \cite{monamanual2001} to perform the underlying conversion. While it is possible that other model checking software (see \cite{belta2017formal} for a list) might be more performant, any such conversion software will run up against the fundamental double exponential complexity of this conversion. 

To mitigate this issue in practice, we note that conversion to an automaton and construction of the GCS can be performed offline if the labeled \hl{regions} and specification are known a-priori. Any initial condition can be used with this offline-computed GCS: only the initial vertex $v_0 = (s_0, q_0)$ will change. Given a GCS, convex optimization solve times tend to be relatively fast (on the order of seconds), even for specifications where the conversion to an automaton was extremely slow. 

A related limitation is the need for labeled convex \hl{regions}. For simple planar scenarios like those in Section~\ref{sec:examples:planar} this may be a reasonable expectation, but for more complex systems with high-dimensional configuration spaces, like the robot arm in Section~\ref{sec:examples:high_dof}, it is not always obvious how to obtain such labeled \hl{regions}. Fortunately, there has been considerable recent progress in decomposition of free space into convex \hl{sets}. Algorithms like IRIS \cite{deits2015computing}, and C-IRIS \cite{amice2023finding}, provide iterative procedures for inflating convex regions of free space. Such algorithms provide a convenient means of constructing labeled \hl{convex sets} for high-dimensional systems. \hl{Efficiently generating convex regions from temporal logic labels directly remains an important area for future research.}

Another drawback of this approach is that we consider path planning rather than trajectory planning. This means that satisfying paths do not consider the natural dynamics of the system and may not be dynamically feasible. The ability of Bezier splines to enforce a desired degree of smoothness, however, allows for dynamical feasibility to be enforced for any differentially flat system, a large class of systems that includes quadrotors \cite{mellinger2011minimum} and differential-drive mobile robots \cite{murray1995differential}.

\hl{From a theoretical perspective, the convex optimization approach provides more limited guarantees than MICP, since convex optimization and rounding is not guaranteed to find a globally optimal solution. In practice, however, this does not appear to be a significant concern. As of the time of writing, we have not come across any examples for which MICP finds a satisfying solution but convex optimization does not. }

Finally, we highlight the limited completeness guarantees available for for full LTL (non-co-safe) specifications (Theorem~\ref{theorem:full_micp_completeness}). This limitation ultimately stems from the difficulty of encoding loop constraints in the GCS framework. Nonetheless, many LTL specifications of practical interest for motion planning either fall within the co-safe fragment (\ref{eq:shortest_path}, \ref{eq:key_door}, \ref{eq:robot_arm}) or are not syntactically co-safe but have solutions with trivial loops (\ref{eq:multitarget}). 

\section{Conclusion and Future Work}\label{sec:conclusion}

We presented a convex optimization solution to temporal logic motion planning. The key idea is to convert the LTL motion planning problem into a shortest path problem in a graph of convex sets. This GCS problem can be solved exactly with mixed-integer programming, but also admits a practical approximate solution via convex optimization and rounding \cite{marcucci2021shortest,marcucci2022motion}.

Our proposed approach scales to complex specifications (Fig.~\ref{fig:large_door_puzzle}) and high-dimensional configuration spaces (Fig.~\ref{fig:robot_arm}), and addresses many of the limitations of standard temporal-logic motion planning methods. \hl{We avoid the clipping and pass-through problems associated with discrete-time temporal logic formulations \cite{lin2014mission} by representing paths with Bezier splines}. Unlike standard MICP-based approaches \cite{raman2014model, belta2019formal, sun2022multi}, our proposed approach does not scale exponentially with the number of sample points used to represent a path (Fig.~\ref{fig:timestep_scalability}). Unlike local optimization \cite{pant2017smooth,mehdipour2019arithmetic,gilpin2020smooth}, learning-based \cite{cai2021reinforcement,LeungPavone2022}, or CBF-based \cite{lindemann2018control} approaches, our proposed approach provides \hl{probabilistic} completeness guarantees under modest assumptions.

Potential areas of future work include extensions to multi-agent and distributed systems, extensions to timed temporal logics like STL and MTL, and use in a model predictive control framework. 

\bibliographystyle{unsrt}
\bibliography{references}

\end{document}

%% file: packages.tex
\usepackage[utf8]{inputenc}
\usepackage[margin=1in]{geometry}
\usepackage[numbers]{natbib}
\usepackage{xcolor}
\usepackage{amsmath}
\usepackage{amsfonts}
\usepackage{amssymb}
\usepackage{bm}
\usepackage{mathtools}
\usepackage{amsthm}
\usepackage{caption}
\usepackage{subcaption}
\usepackage{algorithm}
\usepackage{algpseudocode}
\usepackage{url}
\usepackage{color}
\usepackage{tikz}

%% file: notation.tex
\newcommand{\always}{\square}
\newcommand{\eventually}{\lozenge}
\newcommand{\until}{\mathsf{U}}
\newcommand{\lnext}{\bigcirc}
\newcommand{\true}{\top}

\theoremstyle{plain}
\newtheorem{remark}{Remark}
\newtheorem{definition}{Definition}
\newtheorem{problem}{Problem}
\newtheorem{theorem}{Theorem}

\newcommand{\hl}[1]{#1}